%% file: bdi-space-meth.tex
\documentclass[runningheads]{llncs}

\usepackage{amssymb}
\usepackage{amsmath}
\usepackage{t1enc}
\usepackage[utf8]{inputenc}
\usepackage{ae,aecompl}
\usepackage{framed}
\usepackage{xcolor} \definecolor{shadecolor}{rgb}{0.68,0.6,0.16}
\usepackage{mathptmx}
\usepackage[utf8]{inputenc}
\usepackage{inputenc}
\usepackage[T1]{fontenc}
\usepackage[small,it]{caption}
\usepackage[shortcuts]{extdash}
\usepackage{tikz}
\usetikzlibrary{positioning, shapes, decorations.markings, arrows,backgrounds,calc,automata}
\usepackage{color, soul}
\usepackage{marginnote}
\usepackage{array,booktabs,calc}
\usepackage{stmaryrd}
\usepackage{todonotes}
\usepackage[caption=false]{subfig}
\usepackage{wrapfig}
\usepackage[hyphens]{url}
\usepackage{hyperref}

\usepackage{convenience}
\usepackage{roadlogic}

\usepackage{listings}

\usepackage{siunitx}

\definecolor{greylight}{rgb}{0.7,0.7,0.7}
\definecolor{greyblue}{rgb}{0.0,0.0,0.5}
\definecolor{greyred}{rgb}{0.5,0.0,0.0}
\definecolor{lightgreyblue}{rgb}{0.5,0.5,0.9}
\definecolor{lightgreyred}{rgb}{0.6,0.2,0.2}
\definecolor{lightgreyorange}{rgb}{0.7,0.5,0.2}

\newlength\savedintextsep 

\title{Modular Verification of Vehicle Platooning with Respect to
  Decisions, Space and Time\thanks{Work supported EPSRC
    grants EP/N007565 (Science of Sensor Systems Software), EP/R026092
    (FAIR-SPACE RAI Hub) and EP/L024845/1 (Verifiable Autonomy).}}

\author{Maryam Kamali \and Sven Linker \and Michael Fisher}
\institute{University of Liverpool, UK\\
\email{\{maryam.kamali,s.linker,mfisher\}@liverpool.ac.uk}
}

\begin{document}

\maketitle
\begin{abstract}
The spread of autonomous systems into safety-critical areas has
increased the demand for their formal verification, not only due to
stronger certification requirements but also to public uncertainty
over these new technologies. However, the complex nature of such
systems, for example, the intricate combination of discrete and
continuous aspects, ensures that whole system verification is often
infeasible. This motivates the need for novel analysis approaches that
modularise the problem, allowing us to restrict our analysis to one
particular aspect of the system while abstracting away from
others. For instance, while verifying the real-time properties of an
autonomous system we might hide the details of the internal
decision-making components. In this paper we describe verification of
a range of properties across distinct dimesnions on a practical hybrid
agent architecture. This allows us to verify the autonomous
decision-making, real-time aspects, and spatial aspects of an
autonomous vehicle platooning system. This modular approach also
illustrates how both algorithmic and deductive verification techniques
can be applied for the analysis of different system subcomponents.

\end{abstract}

\input{introduction}
\input{preliminaries}
\input{methodology}

\input{application}
\input{discussion}

\end{document}

%% file: introduction.tex
\section{Introduction}
\label{sec:intro}
\vspace*{-1em}

\noindent Autonomous systems are increasingly being introduced into
safety-critical areas, for example nuclear waste
management~\cite{IntSys:Nuclear:2018}, domestic
robotics~\cite{BeerEtAl2012}, or transportation, in the form of
unmanned aircraft, advanced driver assistance systems, and even
``driverless'' cars.  Although autonomous cars are generally aimed at
increasing the overall safety of traffic, vehicle
\emph{platooning}~\cite{Hsu1994,SolyomC13}, shown in
Fig.~\ref{fig:platooning} in particular provides even more advantages
over single vehicles: it potentially decreases both congestion on
motorways, and fuel consumption, since the relative braking distance
between vehicles should be smaller, and hence the vehicles can make
use of slipstreams with reduced wind resistance. Here, vehicles are
held in sequence on a highway, with distances and speeds controlled by
the platoon rather than the individual vehicle.  Platooning has been
recognised as a valuable means to achieve these goals, and is
encouraged politically, for instance, by the Department of Transport
of the United
Kingdom\footnote{\url{https://trl.co.uk/news/news/government-gives-green-light-first-operational-vehicle-platooning-trial}}.

Autonomous vehicles within a platoon need to be verified to ensure the
overall safety of the platoon. Specifically both autonomous
decision-making concerning leaving/joining the platoon, and low-level
interaction with its environment have to be analysed, in the best case
by providing guarantees for reliable behaviour. To certify the
high-level decisions of an individual autonomous system, the
\emph{rational agent} concept~\cite{Wooldridge2000} is widely used,
since it allows for an analysis of the \emph{reasons} why an
autonomous system chooses a certain action.

The physical interaction of a vehicle with the rest of the platoon of
vehicles in its environment consists of several different
dimensions. Two of the most important are \emph{time} and
\emph{space}. Timing constraints are of major importance to the
overall behaviour of a system. For example, if an unsafe situation is
encountered, the vehicles have to react within a certain time frame to
ensure safety during emergencies. But even for normal vehicle behaviour,
such as joining or leaving a platoon, time constraints are
eminently important~\cite{Burns03}. Spatial aspects are vital for
similar reasons. Ensuring that vehicles do not get too close, or can
fit in the space they are trying to move in to, is clearly important.

So now we reach the key problem. A complex, autonomous system such as
an automotive vehicle platoon, will incorporate a diverse range of
properties and behaviours. If we wish to formally verify \emph{all} of
these dimensions together then we will certainly hit complexity issues
--- multi-dimensional formalisation easily become
\emph{very} complex~\cite{modal:logic:handbook,GKWZ03:book}.  Two
approaches are either to use modular verification
techniques~\cite{Konur2013} or to use abstraction
techniques~\cite{ClarkeGL94} to separate out dimensions of concern. 

\begin{figure}
  \begin{center}
    \begin{tikzpicture}[maneuver]
      \createlanes{2}
      \node[agent, minimum width=26mm] (E_reserve) at ($(begin lane 2) + (.36\linewidth ,0)$) {};
      \node[agent, minimum width=10mm] (E) at ($(begin lane 2) + (.36\linewidth -8mm,0)$) {{\(A\)}};

      \node[agent, opacity=0, minimum width=26mm] (E_inv) at ($(begin lane 1) + (.36\linewidth ,0)$) {};
      
      \draw[dotted] (E_reserve.east) -- (E_inv.east) -- (E_inv.south east) -- (E_inv.south west) 
      -- (E.south west);

      \node[agent, minimum width=15mm] (f1) at ($(begin lane 1) + (.05\linewidth,0)$) {};
      \node[agent, minimum width=10mm] (f1_ps) at ($(begin lane 1) + (.05\linewidth -2.5mm,0)$) {{\(F_4\)}};
      \node[agent, minimum width=15mm] (f2) at ($(begin lane 1) + (.18\linewidth,0)$) {};
      \node[agent, minimum width=10mm] (f2_ps) at ($(begin lane 1) + (.18\linewidth -2.5mm,0)$) {{\(F_3\)}};
      \node[agent, minimum width=15mm] (f3) at ($(begin lane 1) + (.54\linewidth,0)$) {};
      \node[agent, minimum width=10mm] (f3_ps) at ($(begin lane 1) + (.54\linewidth -2.5mm,0)$) {{\(F_2\)}};

      \node[agent, minimum width=15mm] (f3) at ($(begin lane 1) + (.67\linewidth,0)$) {};
      \node[agent, minimum width=10mm] (f3_ps) at ($(begin lane 1) + (.67\linewidth -2.5mm,0)$) {{\(F_1\)}};

      \node[agent, minimum width=20mm] (l) at ($(begin lane 1) + (.82\linewidth,0)$) {};
      \node[agent, minimum width=10mm] (l_ps) at ($(begin lane 1) + (.82\linewidth -5mm,0)$) {{\(L\)}};
    \end{tikzpicture}
    \caption{Vehicle Platooning --- vehicle $A$ joining interior of platoon}
    \label{fig:platooning}
  \end{center}
\end{figure}
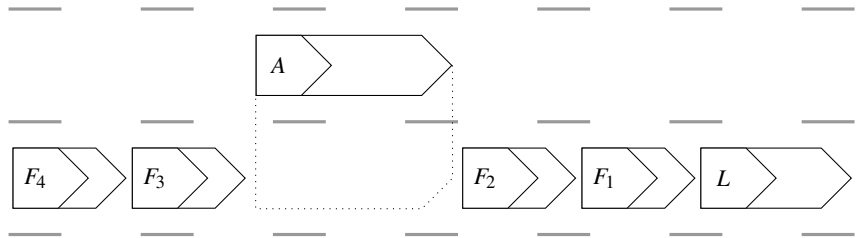

\vspace*{-2em}

\paragraph{Our Approach.} We have identified three key dimensions 
within autonomous vehicle platoons that we wish to assess: autonomous
decision-making, real-time properties, and spatial properties. We also
aim to minimise the change to existing components of the system when
new components are introduced. Consequently, we use abstraction
techniques for the three dimensions, but ensure that verification
results for parts of the system that are unchanged remain valid, and
so the verification task is reduced to checking any new system
components. We show the applicability of this approach by taking an
existing autonomous vehicle platoon system whose decision-making and
real-time properties have already been verified, in~\cite{Kamali2017},
and incorporating spatial aspects. A spatial controller is introduced to
model the lane-changing behaviour of the vehicles in the platoon. This
was something that the original platoon verification
from~\cite{Kamali2017} did not consider and we now show that not only
does the high-level decision making (agent) code remain unchanged, but
since the spatial aspects were shown to be correct
in~\cite{Hilscher2011}, the new verification task is reduced to the
analysis of the real-time requirements.

Consequently, we show how this modular verification approach supports
the flexibility of the underlying hybrid agent architecture, with any
new components of the extended architecture still being verifiable.
The verification of such architectures remains feasible as long as we
can apply appropriate abstraction to the system components.


%% file: preliminaries.tex
\section{Hybrid Agent Architecture}
\label{sec:hybrid}
\vspace*{-1em}

 Cyber-physical systems, such as autonomous vehicles, require 
a sophisticated architecture together with corresponding 
formalism. Practical systems combine continuous environmental
interactions, through feedback control, together with discrete changes
between these control regimes. In traditional hybrid systems, 
separating the high-level decision
making from continuous control concerns is difficult. The 
other drawback of standard hybrid modelling approaches is that the
representation of decision-making can become very complex and
hard to distinguish. We utilise a \emph{hybrid agent architecture}~\cite{LVDFL10:ALCOSP} 
where the decision-making aspect is separated into a distinguished
`agent' while the system still provides for traditional feedback 
control systems. This approach to 
the modelling and development of autonomous systems provides a clear 
separation between these two concerns, and also 
the behaviour of each component is described in much more 
detail that can contribute to reason about their behaviours
separately. Thus, the separation of high-level decision making and 
low-level controllers within a
hybrid agent architecture provides an 
infrastructure for modular verification.

In this paper, we use a hybrid agent architecture, 
proposed for autonomous vehicle
platooning in~\cite{Kamali2017}, as shown in Fig.~\ref{fig:hybrid-arch}.
A \emph{Decision-Making Agent} instructs a \emph{Physical and 
Continuous Engine} by passing instructions through an 
\emph{Abstraction Agent}. The Abstraction Agent receives streams 
of continuous data from the Physical and Continuous Engines, extracts 
discrete information from this, and sends it to the decision-making 
agent. The Physical and Continuous engine manages the real-time 
continuous control of the vehicle through feedback controllers, 
implemented in MATLAB. We assumed that the dynamics of the vehicles are 
continuous, i.e., they may not arbitrarily change positions and velocities. 
An automotive simulator, TORCS
\footnote{The Open Racing Car Simulator
  \url{https://sourceforge.net/projects/torcs}}, is used to implement
the automotive environment and this environment is observed through
the sensory input by the Physical and Continuous engine.

\begin{figure}[htbp]
\begin{center}
\includegraphics[scale=0.4]{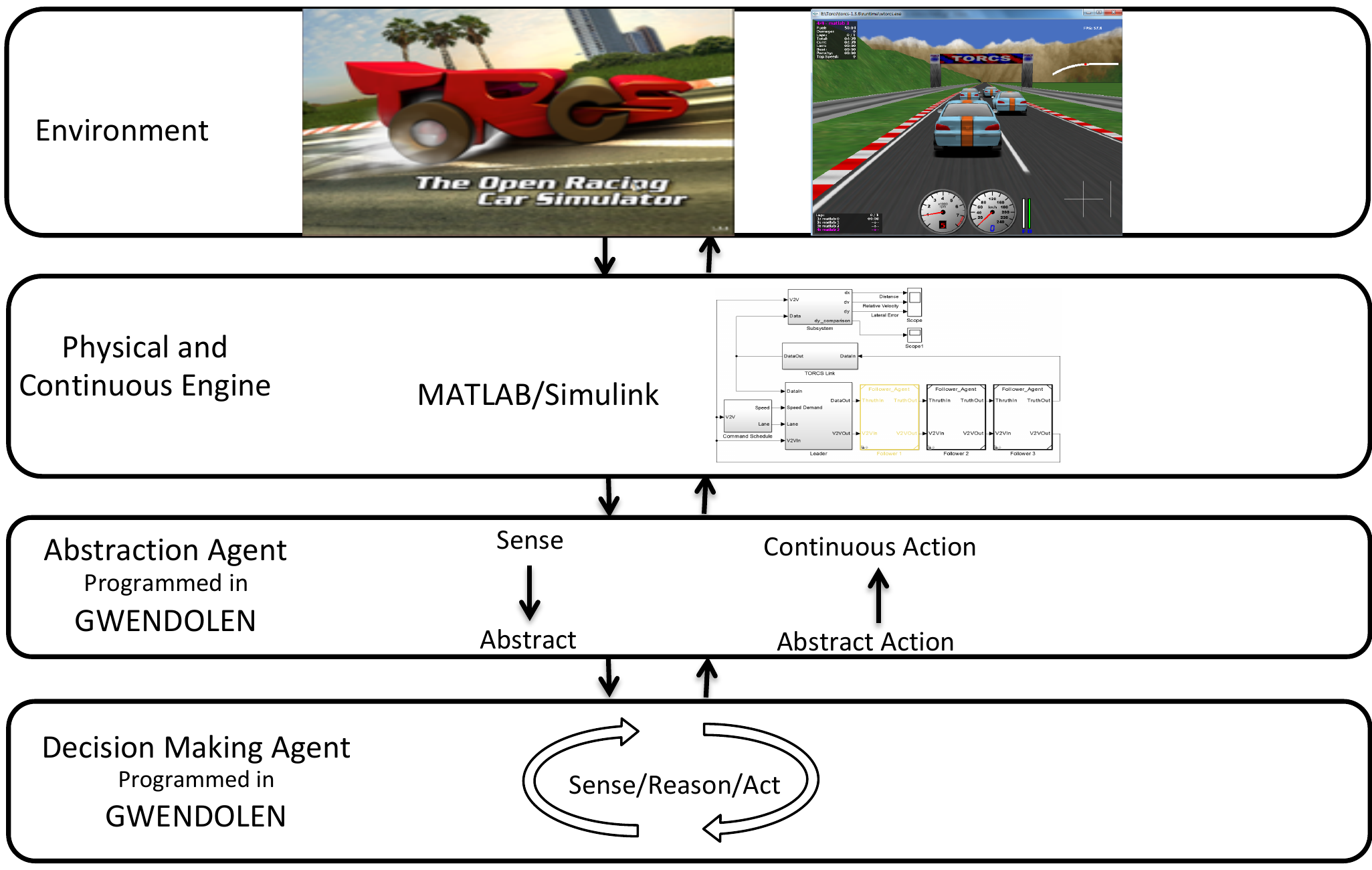}
\caption{Hybrid Agent Architecture~\cite{Kamali2017}}
\label{fig:hybrid-arch}
\end{center}
\end{figure}

The Decision-Making Agent is a \emph{rational agent} \cite{Wooldridge2000}
that not only makes
decisions, but will have explicit reasons for making these decisions.
This allows us to describe \emph{what} the autonomous system chooses
to do, and to reason about \emph{why} it makes its choices. Our Decision-Making
 Agent is based on the BDI (\emph{Belief-Desire-Intention}) paradigm. 
Here, \emph{beliefs} represent the agent's views about the world,
\emph{desires} provide the long-term objectives to be accomplished, and 
\emph{intentions} capture the set of goals currently being undertaken by the 
agent in order to achieve its desires. 

The separation between the Decision-Making Agent and the Physical and
Continuous Engine provides a way to verify the agent behaviour in
isolation from the detail of feedback control. In this work we utilise
\emph{program model-checking} over the Decision-Making Agent. This
allows us to formally verify the \emph{real} agent code rather than a
model of the agent behaviour. This formal verification of agent
behaviour is carried out using the AJPF model checker and the agent
itself is implemented in the verifiable language
\textsc{Gwendolen}~\cite{Dennis2008}. The model-checking approach using AJPF is
used to demonstrate that the BDI agent always behaves according to the
platoon requirements and never intentionally chooses unsafe
options. Unfortunately, model checking of BDI agents through AJPF is
not only resource-heavy, but also lacks support for the formal
verification of timed behaviours.  As indicated above, timing will be
a key principle of relevance to safety-critical behaviour and so, to
tackle this problem, Kamali et al.~\cite{Kamali2017} proposed a modular
approach to the verification of automotive platoons constructed in
this way. They used a combination of AJPF, for internal agent
decisions, and the Uppaal model checker, for global timing behaviours.

We here consider two of the main platooning procedures involved in
joining and leaving a platoon. Both the joining and leaving procedures
are comprised of a series of communications between an individual 
vehicle and the platoon leader aimed to obtain permission to join/leave 
or update the leader when the joining/leaving procedure is 
accomplished. Apart from the required communications, the vehicle switches 
between different controllers, such as moving from `manual' to `automatic'
for speed and steering. 
One of the challenging manoeuvres is changing lanes and
the high-level behaviour of the platoon is verified under the
assumption that the lane changing manoeuvre is carried out safely. In
order to accomplish the fully autonomous platooning while preserving
safety, we extend the previous work of~\cite{Kamali2017} by adding
spatial reasoning to the platooning architecture. Representing
\emph{space} allows us to model the spatial controller of the system
and consequently to verify the safety of the spatial controller
behaviours.

Both the idea, and the concrete definition of the spatial controller,
is taken from previous work~\cite{Hilscher2011}.  The level of this
spatial abstraction is still very high: we do not refer to
specific/metric distances, but instead associate regions of space with
different, abstract, properties.  That is, we distinguish two
different aspects of space needed by a vehicle: its \emph{reservation}
and its \emph{claim}. The intuition here is that the reservation of a
vehicle denotes the part of space that is \emph{necessary} for the
vehicle to operate safely. It comprises both the physical
extent of the vehicle and the distance it needs to come to a
standstill in case of an emergency. The claim, however, is not as
restrictive. It is an additional way for the vehicles to communicate,
similar to the turning signals common to road vehicles. That is, a
vehicle sets a claim somewhere on the motorway to indicate its desire
to occupy this part of the motorway in the (near) future. If the
vehicle decides that changing to the new lane is safe, it mutates its
existing claim into a reservation. Consequently, within our
abstraction the vehicle is considered to be on both lanes at once,
thus modelling the act of changing lanes. For example, in
Fig.~\ref{fig:platooning}, the car \(A\) currently set a claim on the
right lane, to join the platoon.


%% file: methodology.tex
\section{Methodology}
\label{sec:meth}
\vspace*{-1em}

 In this section, we show how the hybrid agent architecture
of Sect.~\ref{sec:hybrid} can be instantiated to verify vehicle
platooning with respect to the agent's decisions, the continuous
behaviour, and the topological spatial changes necessary to change
lanes, e.g., while joining a platoon. To that end, we refine the
instantiation of previous work \cite{Kamali2017} with a new controller
responsible for the spatial aspects of traffic, which in turn is
inspired by previous work of one of the authors \cite{Hilscher2011}.
Generally, our system consists of several controllers, which constrain
the possible behaviour of the vehicles on the road. This implies, in
particular, that the behaviour of the parallel product of two
components is a subset of the behaviour of each single component.  To
show the correctness of our refinement step, we prove a set of proof
obligations including deadlock freedom and invariant preservation. We
also show that all the verified properties of autonomous vehicle
platooning presented in~\cite{Kamali2017}, hold after the refinement
step.

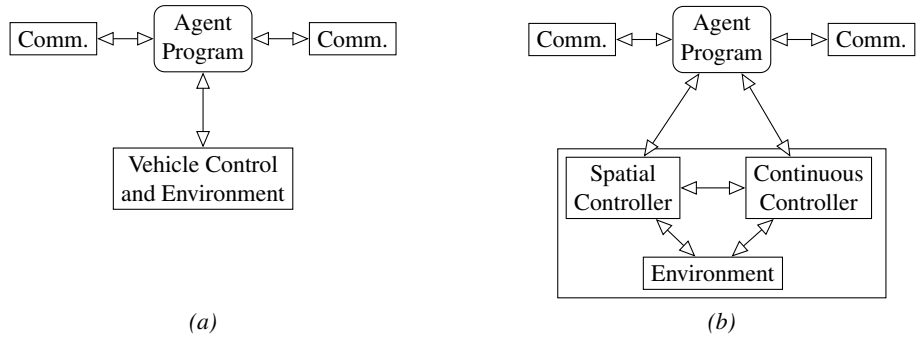
\begin{figure}
  \centering
\subfloat[]{  
\begin{tikzpicture}[align=center, >= triangle 45, fill=white]
  \node[draw,rounded corners ] (agent) {Agent\\ Program};
\node[draw,below =1cm of agent, rectangle] (cont-env) { Vehicle Control\\ and Environment};
\node[draw, right=.75cm of agent, rectangle] (comm2) {Comm.};
\node[draw, left=.75cm of agent, rectangle] (comm1) {Comm.};

\node[below =1cm of agent, rectangle, minimum height=2cm] (cont-env) {};

\draw[<->] (agent) to (comm1);
\draw[<->] (agent) to (comm2);
\draw[<->] (agent) to (cont-env);

\end{tikzpicture}
  \label{fig:orig-arch}
}
\hfill
\subfloat[]{  
\begin{tikzpicture}[align=center, >= triangle 45, fill=white,remember picture]
  \node[draw,rounded corners ] (agent) {Agent\\ Program};
\node[draw,below =1cm of agent, rectangle] (cont-env) {
  \begin{tikzpicture}[remember picture, fill=white]
\node[draw, rectangle] (env) {Environment};
\node[draw, rectangle, above right=.5cm and -.5cm of env] (cont) {Continuous\\ Controller};
\node[draw, rectangle, above left=.5cm and -.5cm of env] (spat) {Spatial\\ Controller};
\draw[<->] (env) to (cont);
\draw[<->] (env) to (spat);
\draw[<->] (cont) to (spat);
  \end{tikzpicture}
};
\node[draw, right=.75cm of agent, rectangle] (comm2) {Comm.};
\node[draw, left=.75cm of agent, rectangle] (comm1) {Comm.};

\draw[<->] (agent) to (comm1);
\draw[<->] (agent) to (comm2);
\draw[<->] (agent) to (cont);
\draw[<->] (agent) to (spat);

\end{tikzpicture}
  \label{fig:ref-arch}
}

  \caption{Original and Refined Architecture}
  \label{fig:arch}
\end{figure}

\noindent Fig.~\ref{fig:arch} shows both the original and refined architecture modelling 
an individual vehicle within a platoon.
The centre of the architecture consists of the \emph{agent program}, which makes autonomous decisions for the vehicle and
may both communicate with other agents via some 
\emph{communication channel}, and with both a \emph{continuous controller} and
an \emph{environment} (cf. Fig.~\ref{fig:orig-arch}). 
A main feature of our approach is a translation of the different components
into simpler abstractions for verification purposes. That is, to verify the
agent program, we can
abstract from the timing aspects of the continuous controller. Thus we  gain
a simple (finite-state) timed automaton as the abstraction of the continuous behaviour.
Similarly, we can reduce the agent program to the few parts necessary for
the communication with the continuous controller for the verification of
the latter. In both cases, the state space is reduced significantly, making
 verification feasible, in the case of the agent program by using
AJPF \cite{Dennis2012} and in the case of the continuous controller by 
using UPPAAL \cite{Behrmann2006}.

\subsection{Agent}
\label{sec:agent}
\vspace*{-.55em}

\noindent The BDI agent program in our architecture is written in
\textsc{Gwendolen} \cite{Dennis2008}, a prolog-style programming
language that incorporates explicit representation of goals, beliefs,
and plans. AJPF is a model checker that accepts \textsc{Gwendolen}
code as an input model.
It allows for the specification and verification of agent properties
with respect to beliefs and intentions. Since the general interface
between the underlying vehicle implementation and the agent is similar
to \cite{Kamali2017}, we could re-use that agent program with only
minor changes. We distinguish between two agent programs: the
\emph{leader}, which manages all joining and leaving requests of
vehicles within, or outside, the platoon, and the \emph{follower},
which defines the functionality of vehicles within the platoon.

We did not need to change the structure of the leader protocol, which is why
we subsequently concentrate on the follower. 
The follower currently implements the interactions for four main features:
\begin{enumerate}
\itemsep=0pt
\item joining a platoon;
\item leaving a platoon;
\item switching the steering control between manual and automatic; and 
\item setting a new distance to the front vehicle.
\end{enumerate}
A vehicle intending to join to a platoon initially sends a joining 
request to the leader and waits for confirmation from the leader.
When it receives the confirmation, it instructs the vehicle to 
change lane and waits for the vehicle to send back a successful confirmation
of changing lane. After receiving the successful confirmation the follower
switches its speed controller to automatic. When the joining vehicle 
is close enough to the proceeding follower within the platoon the agent
instructs the vehicle to switch the steering controller to automatic. Finally,
the joining vehicle confirms the successful joining procedure to the leader.
When the joining vehicle receives a reply back from the leader, it deduces
that the the joining goal has been achieved. The following code shows a 
simplified plan of the agent code for when the joining vehicle switches its 
speed controller from manual to automatic:

\begin{lstlisting}[language=Prolog, label=list: joining-follower, 
 captionpos= b, basicstyle=\scriptsize]

+! joining(X, Y): {B name(X), B join_agreement(X, Y), B changed_lane, 
	~B speed_contr, ~ B steering_contr, ~B joining_distance}
	<- +!speed_contr(1), *joining_distance;
			
\end{lstlisting}

Here: \lstinline{+! joining(X, Y)} indicates the addition of the goal to 
join to the platoon; clauses within $\{\ldots\}$ states the conditions
about the agent's beliefs which must be true such as a join agreement
belief for joining behind the follower $Y$ (\lstinline{B join_agreement}(X, Y)); 
and \lstinline{+! speed_contr(1), *joining_distance}, called the 'body' of 
a plan, is a set of deeds the agent performs for execution of the plan. 
\lstinline{+! speed_contr(1)} adds a new goal for switching the speed 
controller to automatic and \lstinline{*joining_distance} indicates that 
the execution of the plan is suspended until the joining vehicle reduces
the gap with its immediate follower.

Given agent code, one can specify the agent properties with respect to
beliefs, goals, and actions and then verify them using the AJPF model
checker.  As mentioned earlier, due to our modular verification
approach we could skip the re-verification of previous platooning
properties. An example of a safety property is as follows:

\vspace*{-1em}
\begin{small}
\begin{equation} \label{eq:prop4}
\begin{split}
&{\tt \Box\ \left(\ G_{\ X} \ \mathrm{joining}\ (X, Y)\ \& \right. } 
{\tt \left. \lnot B_{\ X} \ \mathrm{join\_agreement}\ (X, Y)\ \right) }\\
&{\tt \rightarrow}
{\tt \ \Box \ \lnot D_{\ X} \ \mathrm{perf}(speed_controller(1))}
\end{split}
\end{equation}
\end{small}
\vspace*{-1ex}

\noindent where $\mathrm{X}$ refers to a joining vehicle that has a goal to
join to a platoon, in front of a platoon follower $\mathrm{Y}$.
$G_{\ \mathrm{X}} \ \mathrm{joining\ (X, Y)}$ indicates a joining goal
that agent $\mathrm{X}$ tries to achieve. $B_{\ X} \ \mathrm{join\_ 
agreement\ (X, Y)}$ indicates the join agreement belief of agent 
$\mathrm{X}$, and $D_{\ X} \ \mathrm{perf(speed\_ controller(1))}$
indicates the action of $\mathrm{perf(speed\_ controller(1))}$ that
agent X performs. This property denotes that if a vehicle never 
believes it has received a confirmation from the leader, then it never 
switches to the automatic speed controller.

\subsection{Continuous Controller}
\label{sec:continuous}
\vspace*{-.55em}

\noindent In the original architecture, we combined the continuous
controller and the environment into one entity. For example, we did
not distinguish between interactions of the agent with the actuators
of the autonomous system and interactions with the human driver. In
both cases, the main feature of the interaction we were concerned with
was the time taken for the controller or environment to react.

As shown in Fig.~\ref{fig:ref-arch}, we now refine the continuous controller and 
the environment into three sub-components. We introduce two controllers,
one referring to the timing aspects and the continuous behaviour of the 
vehicle, and the other specifically to control actuations with respect to space. 

The refinement extends the previous environment with a model of potential collision, which
will be defined in the subsequent section.
It removes the nondeterministic failure of changing lane from the continuous 
controller that implicitly modelled the existence of such a potential collision. A part 
of the continuous controller automaton that has changed in our refinement step is 
shown in Fig.~\ref{fig:cont-aut}. Note that synchronisation channels are changed 
from \(\mathsf{changing\_lane}\) to \(\mathsf{phy\_changing\_lane}\) since the refined continuous 
controller is synchronised with the spatial controller, while in the previous 
controller it was synchronised with the agent automaton. We elaborate on the spatial
controller in the following section.

\begin{figure}
\begin{tabular}{| c | c |}
\hline \includegraphics[scale=0.38]{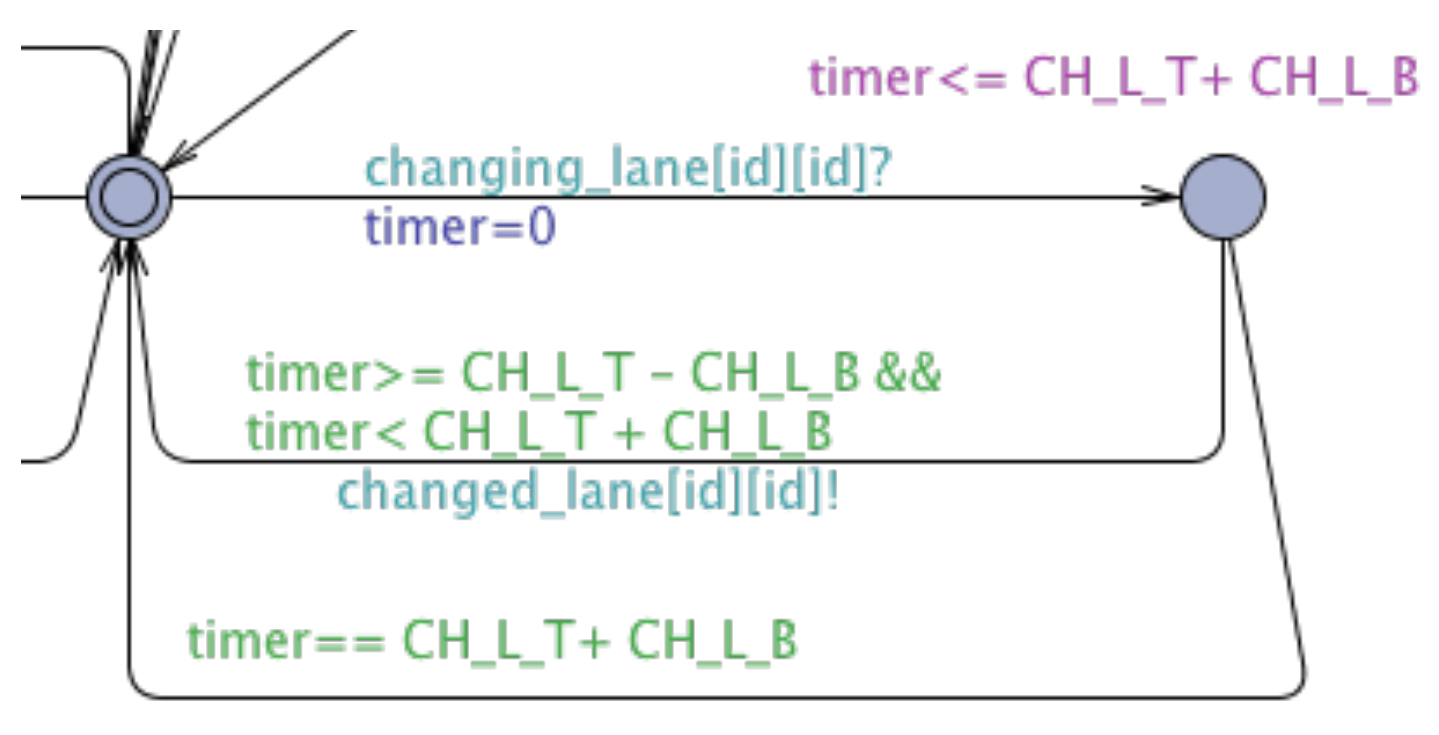}
&
\includegraphics[scale=0.4]{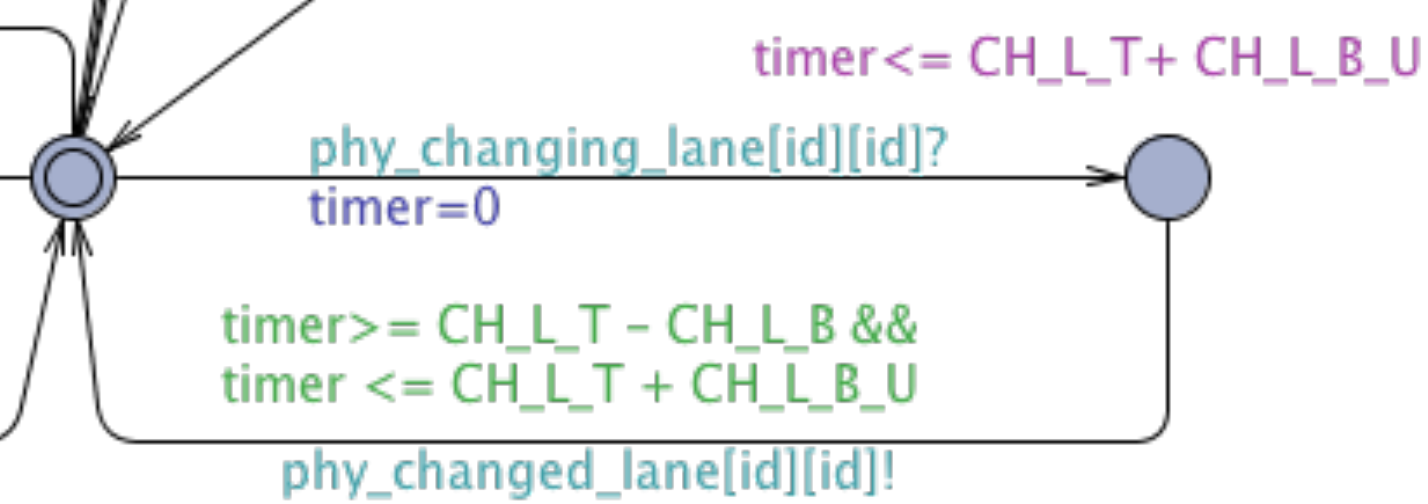}\\
\hline
\end{tabular}
\hspace*{20pt}
\caption{Abstract and Refined Continuous Controller Automata}
\label{fig:cont-aut}
\end{figure}

\subsection{Introducing Space}
\label{sec:space}
\vspace*{-.55em}

\noindent In this section, we present the concrete instantiation of 
the spatial controller, as well as the translation into timed and
untimed automata for verification purposes. To that end,
we formalise the ideas presented in Sect.~\ref{sec:hybrid} on
the spatial model. However, we will not go into all of
the details of the model of space, but refer to previous
work \cite{Hilscher2011,Linker2017}.
 
We 
fix a set of lanes \(\Lanes = \{1, \dots, n\}\) and for simplicity assume
the motorway to be infinitely long. The dimension in the direction of 
the motorway, called the \emph{extension}, is thus modelled by the real numbers
\(\R\). At any point in time, each vehicle \(c\) is then spatially characterised by its 
position \(\pos(c) \in \R\), its physical size \(\ps(c) \in \R\),
its braking distance, i.e., the distance it needs to come to a standstill \(\sd(c) \in \R\),
as well as the lanes it 
reserves \(\res(c) \subseteq \Lanes\) and claims \(\clm(c) \subseteq \Lanes\).
These sets of lanes are subject to certain conditions (e.g., the set of 
claims has to be a singleton and has to be adjacent to the current
reservation, etc.), which we will not expand upon
Each vehicle \(c\) can also perform certain \emph{actions}, in particular
\begin{itemize}
\item[]\begin{description}
\itemsep=0pt
\item[\(\tclaim{c}{n}\):] create a claim on lane \(n\) 
\item[\(\treserve{c}\):] change an existing claim into a reservation 
\item[\(\twdclaim{c}\):] remove/withdraw an existing claim 
\item[\(\twdreserve{c}{n}\):] shrink its reservation to only be on lane \(n\)  
\end{description}
\end{itemize}
While the original defintion of the spatial model allowed for arbitrarily many, even infinite, of
these instantaneous transitions at any point in time, we now restrict the possible transitions such that
after each transition an amount of time greater than zero has to pass. We add this constraint to enforce the permanency of spatial changes on the road.
Subsequently, we will refer to this model of space as \(R\).

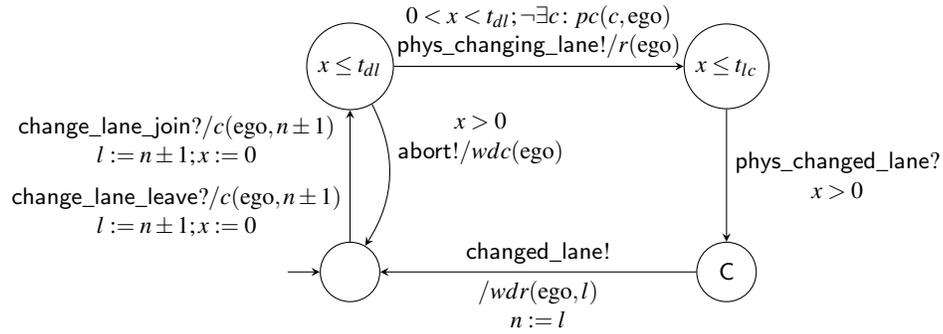
\begin{figure}
  \centering
\begin{tikzpicture}[every text node part/.style={align=center}, >=stealth, initial text ={}]
  \node[state, initial] (init) {};
  \node[state, above =1.75cm of init] (join1) {\(x \leq t_{dl}\)};
  \node[state, right=4.2cm of init] (wd1) {\textsf{C}};
  \node[state] (join2) at (join1 -| wd1) {\(x \leq t_{lc}\)};

  \draw[->] (init) to node[ left, xshift=0cm] {\(\mathsf{change\_lane\_join?}/\tclaim{\ego}{n\pm1}\)\\\(l :=n\pm 1; x:=0\)\\[.5em] \(\mathsf{change\_lane\_leave?}/\tclaim{\ego}{n\pm 1}\)\\\(l :=n \pm 1; x:=0\)} (join1);
  \draw[->, bend left = 30] (join1) to node[ right, yshift=.5cm] {\(x > 0 \)\\ \(\mathsf{abort!}/\twdclaim{\ego} \)} (init);
  \draw[->] (join1) to node[above] {\(0 < x < t_{dl}; \lnot \exists c \colon \pcc{c, \ego}\) \\ \(\mathsf{phys\_changing\_lane!}/ \treserve{\ego}\)} (join2);
  \draw[->] (join2) to node[right] {\(\mathsf{phys\_changed\_lane?}\)\\ \(x > 0\)} (wd1);
  \draw[->] (wd1) to node[above] {\(\mathsf{changed\_lane!}\)} node[below]{\(/\twdreserve{\ego}{l}\)\\\(n :=l\)} (init);
\end{tikzpicture}
  \caption{Spatial Controller for Joining and Leaving a Platoon (\(t_{dl} < t_{lc}\))}
  \label{fig:spat-cont}
\end{figure}

Using these abstract definitions as the semantics, we defined a
specification logic with an emphasis on multi-lane traffic
\cite{Hilscher2011}. However, in this work we will not require the
full logic, and hence we only explain the necessary details. One main
feature of the logic is that it employs \emph{local reasoning}, that
is, a formula is evaluated with respect to a finite part of the
motorway, as perceived by a distinguished vehicle, which is sometimes
referred to as the \emph{ego} vehicle.  This finite perception is
called the \emph{view} of a vehicle, and consists of a subset of the
lanes, as well as a finite interval of the real numbers, the extension
of the view.  We employ two spatial atoms \(\reserved{c}\) and
\(\claimed{c}\), which denote that the current view consists of a
single lane and a non-empty extension, and is fully occupied by the
reservation (claim, respectively) of \(c\). Furthermore, we use a
single modality \emph{somewhere} \(\somewhere{\varphi}\), which
denotes that the formula \(\varphi\) holds somewhere on the space
under consideration.  With these specific definitions, and standard
first-order operators, we can express the following two formulas.
\begin{align*}
cc & \equiv \lnot \exists c \colon c \neq \ego \land \somewhere{\reserved{\ego} \land \reserved{c}}\\
  \pcc{c, \ego} & \equiv c \neq \ego \land \somewhere{\claimed{ego} \land (\reserved{c} \lor \claimed{c})} 
\end{align*}
Formula \(cc\) denotes the existence of a vehicle \(c\) whose reservation
overlaps with the reservation of \(\ego\). According to our
explanations above, this would amount to an unsafe situation. For
simplicity, we term such situations as \emph{collisions}, even though
 \(c\) may only have encroached upon the braking distance of
\(\ego\) or vice versa.  Formula \(\pcc{c, \ego}\) denotes that the
claim of \(\ego\) overlaps with either the claim of \(c\) or its
reservation. This may result in an unsafe situation, if \(\ego\)
changed its claim into a reservation. Hence, \(\pcc{c, \ego}\) allows
us to identify potentially unsafe situations, and so take measures to
mitigate this.

To model the spatial behaviour of a vehicle joining or leaving the platoon, we
will use a type-amended timed automata called \emph{automotive-controlling timed automata} (ACTA) \cite{Hilscher2016}.
These augment timed automata with the possibility to use spatial formulas as guards and invariants, as 
well as to use the spatial actions described above at the transitions.  
Figure~\ref{fig:spat-cont} shows the controller in terms of an ACTA, where \(\ego\) refers
to the vehicle the controller is implemented in. The upper part of the controller is concerned with
the vehicle trying to join a platoon, while the lower part handles the leaving of a platoon. The actions
 \(\mathsf{change\_lane\_join}\), \(\mathsf{change\_lane\_leave}\), \(\mathsf{changed\_lane}\), 
and \(\mathsf{abort}\) are used to communicate with the decision making agent. The first two actions
are used by the agent to initiate the corresponding manoeuvre, while the spatial
controller uses \(\mathsf{changed\_lane}\) and \(\mathsf{abort}\) to indicate a successful and
unsuccessful lane-change manoeuvre, respectively. 
The channel \(\mathsf{phys\_changed\_lane}\) is a direct
communication link with the continuous controller, which indicates that steering
onto the new lane was successful. Observe that we chose to encode the necessary delays after
the transitions into this controller as well via the clock \(x\). 

For the verification of the other components, we need to provide abstractions from 
the ACTA given above into both an untimed automaton, and a standard
timed automaton. 
To abstract from both the timing and spatial definitions, we 
 remove all references to clocks, spatial formulas and
spatial actions, i.e., we only keep the discrete actions, and therefore maintain
 the order of actions.
In this way, we create a simple finite automaton which serves as the abstraction of the
spatial controller that can only be used during the verification of the agent programs.
The translation into timed automata is slightly more involved.
We employ a global set of identifiers for each vehicle. In fact, this set was already 
used to identify the different vehicles by parameterising the continuous controllers \cite{Kamali2017}. 
Hence, we replace each occurrence of \(\ego\) with the parameter \(id\). Furthermore, we introduce
a global array \(c\) of Boolean values, where the identifiers serve as the indices, and
each entry denotes whether the corresponding vehicle currently possesses a claim.
Whether a vehicle is currently engaged in a lane-change manoeuvre, i.e., whether
it uses two lanes at once for its reservation, is indicated by a variable \(r\), which is
local to each controller. 

However, since claims and reservations are strongly tied together,
we also need to define an abstraction of the road's behaviour. To that end, we
chose to use a very simple abstraction: a potential collision can only happen,
if at least one vehicle currently holds a claim. Furthermore, a potential collision
has to last an arbitrary amount of time greater than zero before it can be resolved. 
This is a result of the assumption on the vehicles dynamics to be continuous and the
necessary delays after the spatial transitions. Note that
a potential collision can happen due to two reasons: either a spatial transition or
the different velocities of two cars. In both cases, our model and its assumptions
ensure that the situation persists for a non-zero amount of time. 
We can formalise these two properties with the following abstraction of the road's behaviour,
as shown in Fig.~\ref{fig:abs-road}. In this figure, \(y\) is a clock used to enforce
the timing behaviour.

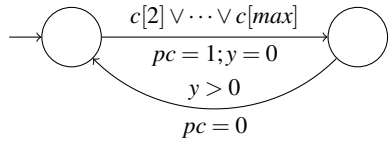
\begin{figure}
  \centering
  \begin{tikzpicture}[initial text={}]
    \node[state, initial] (init) {};
    \node[state, right=3cm of init] (pc) {};

    \draw[->] (init) to node[above] {\(c[2] \lor \cdots \lor c[max]\)} node[below] {\(pc = 1; y = 0\)} (pc);
    \draw[->, bend left=45] (pc) to node[above] {\(y > 0\)} node[below] {\(pc = 0\)} (init);
  \end{tikzpicture}
  \caption{Abstraction of Spatial Behaviour on the Road}
  \label{fig:abs-road}
\end{figure}

With these changes, the timed abstraction of the spatial controller is as 
shown in Fig.~\ref{fig:timed-abs-spat-cont}.
The timing behaviour is exactly as in
the original automaton. Hence, if we can verify the other controllers in the presence
of this controller, we can guarantee the safety of the overall system. 
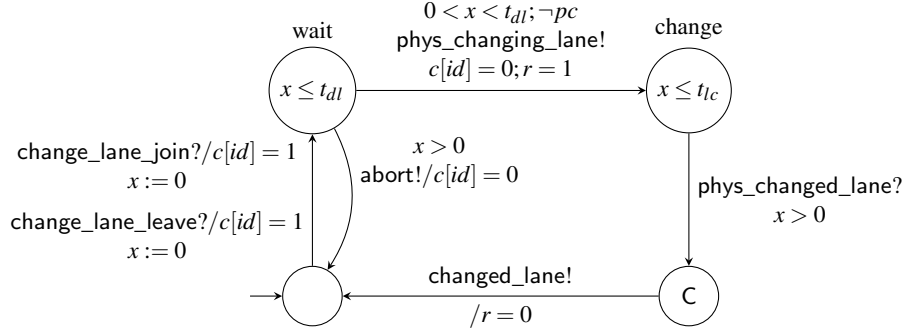
\begin{figure}
  \centering
\begin{tikzpicture}[every text node part/.style={align=center}, >=stealth, initial text ={}]
  \node[state, initial] (init) {};
  \node[state, above =1.75cm of init, label={above:wait}] (join1) {\(x \leq t_{dl}\)};
  \node[state, right=4.2cm of init] (wd1) {\textsf{C}};
  \node[state, label={above:change}] (join2) at (join1 -| wd1) {\(x \leq t_{lc}\)};

  \draw[->] (init) to node[ left, xshift=0cm] {\(\mathsf{change\_lane\_join?}/c[id] = 1\)\\\(x:=0\)\\[.5em] \(\mathsf{change\_lane\_leave?}/c[id]=1\)\\\( x:=0\)} (join1);
  \draw[->, bend left = 30] (join1) to node[ right, yshift=.5cm] {\(x > 0\)\\\(\mathsf{abort!}/c[id] = 0 \)} (init);
  \draw[->] (join1) to node[above] {\(0 < x < t_{dl}; \lnot pc\) \\ \(\mathsf{phys\_changing\_lane!}\)\\\( c[id]=0;r=1\)} (join2);
  \draw[->] (join2) to node[right] {\(\mathsf{phys\_changed\_lane?}\)\\ \(x>0\)} (wd1);
  \draw[->] (wd1) to node[above] {\(\mathsf{changed\_lane!}\)} node[below]{\(/r=0\)} (init);
\end{tikzpicture}
  \caption{Timed Abstraction of Spatial Controller of Fig.~\ref{fig:spat-cont} (\(t_{dl} < t_{lc}\))}
  \label{fig:timed-abs-spat-cont}
\end{figure}

Finally, we need to define how the agent program and the continuous
controller can be abstracted for the verification of the spatial
properties. To that end, observe that the specification logic for the
spatial properties does not contain modalities to refer to timings or
decisions of the agent. That is, for the spatial properties, we do not
refer to either time constraints or the specific goals or intentions
of the agent. Hence, for spatial verification, we use the untimed
automaton of the continuous controller which was also used during the
verification of the agent. Similarly, the untimed abstraction
automaton of the agent program used in the verification of timing
aspects can be re-used during the verification of spatial properties.

\begin{lemma}
  Let \(A_i\), \(V_i\) and \(S_i\) be the agent program, continuous controller and
spatial controller, respectively, of vehicle \(i\), with \(i \in \{1,2\}\). Furthermore
let \(\mathit{Comm}12\) be the component modelling the communication of vehicle \(1\) and \(2\).
Let \(S_i^\prime\) be the abstractions of the spatial controllers (Fig.~\ref{fig:timed-abs-spat-cont}), 
\(A_i^\prime\) the abstractions of the agent programs, 
\(R^\prime\) the abstraction of the road (Fig.~\ref{fig:abs-road}), and \(\varphi_t\) a formula describing a time property.
If 
\(A_1^\prime \| V_1 \| S_1^\prime \| \mathit{Comms}12 \| A_2^\prime \| V_2 \| S_2^\prime \| R^\prime \models \varphi_t \)
then
 \(A_1 \| V_1 \| S_1 \| \mathit{Comms}12 \| A_2 \| V_2 \| S_2 \| R\models \varphi_t\).  
\end{lemma}
\begin{proof}[Sketch]
  The timing behaviour of \(A_i\) and  \(A_i^\prime\) is the same (cf. \cite{Kamali2017}). Furthermore,
  the timing constraints on \(S_i\) and \(S_i^\prime\) are also the same. Now, after each spatial transition,
  in the original \(S_i\), some time has to pass. In both \(S_i\) and its abstraction \(S_i^\prime\), every time
  a clock is reset
the guards on the outgoing transitions of the target state \(s\) require the automaton to stay in \(s\) for some
time.
  Finally, if the abstraction \(R\) flags a potential collision, then the original system possesses a 
  trace containing at least one claim for a vehicle. 
Let us assume this claim is of vehicle \(2\). Then, all possible traces starting from this configuration
are also possible in the abstraction \(R\).  Hence, whenever we can prove that the abstraction
satisfies a timed formula \(\varphi_t\), the original system also satisfies \(\varphi_t\).   
\qed
\end{proof}

\begin{lemma}
 Let \(A_i\), \(V_i\) and \(S_i\) be the agent program, continuous controller and
spatial controller, respectively, of vehicle \(i\), with \(i \in \{1,2\}\). Furthermore
let \(\mathit{Comm}12\) be the component modelling the communication of vehicle \(1\) and \(2\).
Now let \(A_i^\prime\) and \(V_i^\prime\) be the abstractions without references to spatial
properties as described above.
Then, if 
\(A_1^\prime \| V_1^\prime \| S_1 \| \mathit{Comms}12 \| A_2^\prime \| V_2^\prime \| S_2 
\| R^\prime\models \varphi_s \)
then
 \(A_1 \| V_1 \| S_1 \| \mathit{Comms}12 \| A_2 \| V_2 \| S_2 \| R\models \varphi_s\).  
\end{lemma}
\begin{proof}[Sketch]
  Again, this holds since the abstractions \(A_i^\prime\) and \(V_i^\prime\) allow for more
behaviour than the original automata. Furthermore, the spatial properties may neither
refer to the internals of the agent program, nor to time aspects of the system. \qed
\end{proof}


%% file: application.tex
\section{Verification of Vehicle Platooning}
\label{sec:verification}
\vspace*{-1em}
 In this section, we explain the verification approach built
on the methodology presented in Sect.~\ref{sec:meth}.\footnote{The model and
the verified properties 
can be found at \url{https://github.com/VerifiableAutonomy/AgentPlatooning}} On one hand, we
did not have to re-run most of our verification methods from our
previous work, particularly running AJPF, since we only refined
non-agent parts of the system. On the other hand, we needed to show
that the refinement step was valid by proving proof obligations. In
the following, we first identify a set of proof obligations that we
proved to verify our refinement step. We then denote the spatial
properties that we checked for our concrete vehicle platooning. We
also point out those parts of the system that remained unchanged and
consequently not re-verified. Finally, we prove that the spatial
controller which is added to our concrete vehicle platooning is a safe
fragment of the space model in~\cite{Hilscher2011}.

\subsection{Proof Obligations}
\label{sec:po}

\noindent The refinement step allows us to introduce more details
about the spatial properties of vehicle platooning. However, we need
to ensure that the new details do not violate the system invariants,
and do not introduce deadlocks. The checking needs to be considered for
both verification of agent and timing behaviours.  The untimed
abstraction of the spatial controller only allows the same set of
sequences of interactions with the agent. This means that we did not
change the \emph{structure} of the agent programs themselves; neither
the leader nor the follower, i.e., the refinement step is correct
wrt. agent behaviour.  To discuss the correctness of our refinement
step wrt. temporal behaviour, we check four main proof obligations,
shown in Fig.~\ref{table:po}. The first three obligations are verified
using the Uppaal model checker, followed by a discussion of the
correctness of the fourth obligation.  We instantiated the agent timed
automata, spatial, and continuous controllers for a platoon of four
vehicles and one leader. We choose an arbitrary vehicle, for example
vehicle \(2\), to denote our proof obligations and properties of
interest, and described these with respect to this vehicle. Note that
\(a2\) is the follower agent program as implemented in vehicle \(2\)
and \(s2\) is the lane-change (spatial) controller of the same
vehicle.

 \begin{figure}\begin{center}\begin{small}
\begin{tabular}{| c | l | c |}
\hline Deadlock Freedom & $A\Box \  not \ deadlock $\\
\hline Possible to join and leave & $E\lozenge \ a2.join\_completed  $\\
& $E\lozenge \ a2.leave\_completed   $\\
\hline Time bound for joining and leaving & $A\Box \ a2.join\_completed \ \mathbf{imply}$\\ 
& $(a2.process\_time>=50 \land a2.process\_time<90) $ \\ 
& $A\Box \ a2.leave\_completed \ \mathbf{imply}$\\ 
& $(a2.process\_time>=30 \land a2.process\_time<50) $ \\ 
\hline No new communication transaction & changes were restricted to continuous and spatial\\
& controllers \\
\hline
\end{tabular}
\end{small}
\caption[Caption for LOF]{Proof Obligations, with formalisation in timed temporal logic
  \footnotemark}
\label{table:po}
\end{center}
\end{figure}
\footnotetext{$A$=``all paths''; $E$= ``exist a path''; $\Box$=``Always''; $\lozenge$=``Eventually''}

The first proof obligation that we verified was deadlock freedom. We
showed that our refinement step was not too restrictive. The second
proof obligation ensures that adding the spatial controller does not
decrease the functionality of the platooning, and we checked whether
joining and leaving procedures can occur.  In the previous Uppaal
model of platooning, we assumed that change lane could happen in $20
\pm CH\_L\_B$ where $CH\_L\_B$ was reflecting the uncertainty of the
changing lane. In our refinement, the lower bound remains the same,
however, the upper bound splits to two waiting times for free space
$t_{dl}$ (cf. Fig.~\ref{fig:spat-cont}) and the uncertainty of the
changing lane ($CH\_L\_B - t_{dl}$). Therefore, we could show that the
time bound of joining and leaving remain the same as the previous
model (The third proof obligation in Fig.~\ref{table:po}).

In the refinement step,
 we defined two new channels  representing the communication between the spatial
 controller
and
 continuous controller, \(\mathsf{phy\_changing\_lane}\) and
\(\mathsf{phy\_changed\_lane}\). As these channels are not used in any other parts of
the system, we can guarantee that no new communication transition 
is added to
any other part of our model.

\subsection{Spatial Properties of Vehicle Platooning}
\vspace*{-.55em}
 We can verify that if a vehicle requested a
lane-change, i.e., the spatial controller reaches the \emph{wait}
state (cf. Fig.~\ref{fig:timed-abs-spat-cont}), and still perceives a
potential collision after the waiting time \(t_{dl}\), then the
corresponding manoeuvre in the agent program will fail. 
\begin{align}
(s2.wait \land pc \land s2.x == t_{ld})   \longrightarrow (a2.failed\_to\_join \lor a2.failed\_to\_leave)
\end{align}
In this formula \(\longrightarrow\) denotes the ``leads-to'' operator
of Uppaal.
 Observe that we cannot identify whether the join manoeuvre
or the leave manoeuvre failed, since the spatial controller acts
similarly for both manoeuvres. Note that identification of the
manoeuvre can be easily implemented by adding a flag to the spatial
controller automaton.  We can also show that, whenever the spatial
controller chooses that a lane-change can be safely initiated, it does
not perceive a potential collision on the road. Furthermore, as long
as it stays in this state, no potential collision can arise.
\begin{align}
  A\Box \lnot (s2.change \land pc)
\end{align}
 This
property shows that the space on the road as formalised in Fig.~\ref{fig:abs-road} 
is ``well-behaved'' within this abstraction,
since a potential collision can only happen, if a vehicle possesses a claim. However,
if the controller of vehicle \(2\) is in state \(change\), it already changed its claim
to a new reservation.  The time needed to verify these properties was similar to
the time needed for the proof obligations, which, compared to our
previous attempt, changed negligibly.

\subsection{Spatial Safety Property}
\label{sec:spat-safe}
\vspace*{-.55em}
 The main property that the spatial controller must ensure is that the space
used by two different vehicles is disjoint. That is, it has to ensure
that the formula \(cc\) as shown in Sect.~\ref{sec:space} is an invariant
of the system.
To that end, 
we re-use  a verification
result \cite{Linker2017} of a much more general controller specification encoded 
in the theorem prover Isabelle/HOL \cite{Nipkow2002}. 
Safety in this work means that \(\forall e \colon \mathsf{safe}(e)\) is a global invariant,
where \(\mathsf{safe}(e)\) is defined as follows.
\begin{align*}
 \mathsf{safe}(e) \equiv   \Box \lnot \exists c \colon c \neq e \land \somewhere{\reserved{e} \land \reserved{c}} 
\end{align*}
Observe that the modality \(\Box\) is similar in intent to the box-modality used in the 
previous sections, in that it quantifies over arbitrary transition sequences, but it does not
allow us to specify timing constraints.
To prove this property to be an invariant, we need two main assumptions:
\begin{enumerate}
\item All vehicles keep their distance to the vehicles in their front and back
\item All vehicles adhere to a certain protocol for changing lanes with respect to the platoon under consideration
\end{enumerate}
We do not elaborate on the first assumption. However,
the second assumption is that, the vehicle must not mutate its claim
into a reservation, in case of a potential collision during the phase
where a claim is held.
Formally, we have the following constraint, where
\(\Box_{\mathsf{r}(d)}\) quantifies over the transition where the
vehicle \(d\) changes its claim into a reservation and \(c\) ranges
only over the vehicles within the platoon.
\begin{align*}
  \mathsf{LC} &\equiv \forall d \colon  \exists c \colon \pcc{c,d} \implies \Box_{\mathsf{r}({d})} \bot
\end{align*}
 For simplicity, assume that the platoon under consideration
consists of two vehicles as in Sect.~\ref{sec:meth}.  That is, the
platoon \(P\) consists of the following components.
\begin{align*}
  P &\equiv A_1 \| V_1 \| S_1 \| \mathit{Comms}12 \| A_2 \| V_2 \| S_2 \| R
\end{align*}
Now, let \(\sem{S_1}\) and \(\sem{S_2}\) be the possible behaviours allowed by the 
controllers \(S_1\) and \(S_2\) as presented in Sect.~\ref{sec:space}. Since
the only transition to change a claim into a reservation is guarded by the 
potential collision check,
 we have for \(i \in \{1,2\}\),
\begin{align*}
 \sem{S_i} \cap \{tr \mid tr \models \exists c \colon  \pcc{c,i} \land \lozenge_{\mathsf{r}(i)} \top\} = \emptyset
\end{align*}
That is, since the behaviour of the parallel product of \(S_1\) and \(S_2\) is a subset of both \(\sem{S_1}\) and \(\sem{S_2}\), we also get 
\begin{align*}
  \sem{S_1 \| S_2} \cap  \{tr \mid tr \models \exists c \colon \pcc{c,i} \land \lozenge_{\mathsf{r}(i)} \top\} = \emptyset \enspace.
\end{align*}
Since the other controllers may only further restrict the possible behaviour of the
platoon, we also have 
\begin{align*}
  \sem{P} \cap  \{tr \mid tr \models \exists c \colon \pcc{c,i} \land \lozenge_{\mathsf{r}(i)} \top\} = \emptyset\enspace. 
\end{align*}
 Due to our assumption on the behaviour of all other vehicles, we can infer that \(\sem{P}\) does not 
contain any traces where other vehicles create a reservation during a potential collision.
Hence, we can strengthen this property even further.
 \begin{align*}
   \sem{P} \cap  \{tr \mid tr \models \exists c,d \colon \pcc{c,d} \land \lozenge_{\mathsf{r}(d)} \top\} = \emptyset \enspace,
 \end{align*}
which in particular implies
\(  \sem{P} \subseteq S\).
This yields
\(  \sem{P} \models \mathsf{LC}\),
which has been shown to ensure that \(P \models \forall e \colon \Box \mathsf{safe}(e)\).
Hence, our controller is a refinement
of the general case, which was shown to be safe. 


%% file: discussion.tex
\section{Concluding Remarks}
\label{sec:discuss}
\vspace*{-1em}

\paragraph{Contribution.}
We presented a verification technique for autonomous systems
based on the use of a hybrid agent architecture.  The decomposition of
concerns inherent in this architecture allows us to define different
aspects of the system within different formalisms, which are tied
together by the communication structure of the system and its timing
constraints. For each of the formalisms we defined a translation into
an abstraction compatible with the other formalisms. In this way, we
can concentrate on each aspect in turn during verification, which both
reduces the state space, and allows us to use different techniques for
each aspect.

Decomposition techniques often isolate the single components and replace 
the interaction with the other components by assumptions, which are then shown to 
be guaranteed by the respective components \cite{Misra1981}.
In contrast, during each step of the verification, 
we keep the general structure of the overall system. That is, we do not really \emph{decompose}
the system, but abstract from different parts during each step. This eliminates the need
to infer the behaviour of the single components, e.g., in the form of guarantees. Of course,
this also means that large parts of the state space are retained during verification, in comparison
to techniques which replace the other components with the guarantees they keep. However, 
we have shown that our approach is both feasible for autonomous systems, as well as that
it scales well if new aspects are to be verified. 

\paragraph{Related Work} M\"{u}ller et al. presented a technique to verify safety of hybrid systems 
\cite{Muller2016} based on the identification of components. In their
approach, they need to define and verify contracts for the behaviour
of each component, which may simply assumed to be true during the
verification of other components.  In this manner, they can reduce the
verification task for each component. Their systems need to be defined
within a single formalism, differential dynamic logic
\cite{Platzer2010}, and are verified with the distinguished tool
KeymaeraX \cite{Fulton2015}. In contrast, we can rather easily
incorporate new formalisms into our approach, as evidenced by the
introduction of the lane-change controller and the necessary spatial
formalism. This is due to the minimised interaction between our
controllers, in the form of time and communication. In this way, we
may use the verification techniques suitable for the corresponding
subsystems, as long as we have a sensible abstraction and refinement
results for each system.

Abstraction and refinement techniques are often employed for verification purposes. 
For example, in counter-example guided abstraction (CEGAR) \cite{Clarke2000}, the 
verification starts with a very broad abstraction. If a counter-example is found to be 
spurious, i.e., it is not viable in the original system, the corresponding part of the 
state-space has to be refined to eliminate this example. That is, the system is 
analysed in a top-down fashion, from the broadest possible abstraction to an explicit 
description of the system. Our method proceeds in a somewhat orthogonal way. 
Instead of building an abstraction of the system as a whole, we build several 
abstractions according to the type of property we intend to verify.

The concept of abstraction and refinement relations is used prominently in  Event-B \cite{Abrial2010}.
Banach and Butler have used a hybrid extension of Event-B to model and verify controllers
used in autonomous driving \cite{Banach2013,Banach2014}. However, these controllers have
been verified on their own, and the interactions between them have not been verified.
